\pdfoutput=1

\documentclass[conference]{IEEEtran}

\usepackage{mathptmx}
\usepackage{amsmath}
\usepackage{amsthm,amsmath,amssymb}
\usepackage{bm} 
\usepackage{amsfonts}
\usepackage{subeqnarray}
\usepackage{cases,cite}
\usepackage[utf8]{inputenc}
\usepackage{graphicx}
\usepackage{amssymb}
\usepackage{float}
\usepackage{graphicx} 
\usepackage{subfigure}
\usepackage{algpseudocode}
\usepackage{url}
\newtheorem{proposition}{Proposition}

\newtheorem{lemma}{Lemma}

\IEEEoverridecommandlockouts

\begin{document}

\title{Consensus-Based Transfer Linear Support Vector Machines for Decentralized Multi-Task Multi-Agent Learning}

\author{\IEEEauthorblockN{Rui Zhang}
\IEEEauthorblockA{Department of Electrical and Computer Engineering\\
New York University, Brooklyn, NY, 11201\\
Email: rz885@nyu.edu\\}
\and
\IEEEauthorblockN{Quanyan Zhu}
\IEEEauthorblockA{Department of Electrical and Computer Engineering\\
New York University, Brooklyn, NY, 11201\\
Email: qz494@nyu.edu}}

\maketitle

\begin{abstract}
Transfer learning has been developed to improve the performances of different but related tasks in machine learning. However, such processes become less efficient with the increase of the size of training data and the number of tasks. Moreover, privacy can be violated as some tasks may contain sensitive and private data, which are communicated between nodes and tasks. We propose a consensus-based distributed transfer learning framework, where several tasks aim to find the best linear support vector machine (SVM) classifiers in a distributed network. With alternating direction method of multipliers, tasks can achieve better classification accuracies more efficiently and privately, as each node and each task train with their own data, and only decision variables are transferred between different tasks and nodes. Numerical experiments on MNIST datasets show that the knowledge transferred from the source tasks can be used to decrease the risks of the target tasks that lack training data or have unbalanced training labels. We show that the risks of the target tasks in the nodes without the data of the source tasks can also be reduced using the information transferred from the nodes who contain the data of the source tasks. We also show that the target tasks can enter and leave in real-time without rerunning the whole algorithm. 
\end{abstract}
\section{Introduction}
Machine learning algorithms are largely used nowadays in various areas, e.g., face detection \cite{osuna1997training} and search engines \cite{mccallum1999machine}. Traditionally, machine learning makes predictions or classifications based on the assumption that the training and the testing data come from the same source or distribution \cite{shao2015transfer}. However, this assumption may not hold in many real applications\cite{pan2010survey}; for example, the training data can be outdated, or insufficient to build a good classifier. In such cases, it is difficult to find the classifier using traditional machine learning frameworks. 

Recent researches on transfer learning provide a solution to address such problems. It has been shown that machine learning tasks can benefit from other similar tasks by knowledge transfer \cite{shao2015transfer,pan2010survey}. For instance, web-page data can become outdated easily as the web content changes frequently, and new training data are expensive to acquire as the labeling of the data is costly. Since parts of the outdated data still contain useful information,  knowledge can be transferred from them to train a classifier together with the new data\cite{dai2007boosting}. 

Although the knowledge transfer can improve the performance of machine learning, the training process using a large amount of data is often not efficient. For traditional transfer learning, training data are communicated between tasks\cite{evgeniou2004regularized}. The direct data sharing is not possible when the volume of the data is huge and they contain private information. For example, training data may come from different nodes of a wireless sensor network (WSN), and their communication with a fusion center can be either costly or restricted due to scalability, privacy or power limitations \cite{forero2010consensus}. 

This paper aims to address this issue by extending transfer learning into a distributed framework in the context of support vector machines (SVMs) illustrated in Fig. \ref{fig:DTSVM}. The framework trains different but related tasks together with linear SVMs at each node in a fully distributed network. The decision variables to classify testing data are found by minimizing the regularized errors of training data of each task. One set of consensus constraints is introduced to force all the tasks to share the same terms of decision variables at each node while another set of consensus constraints is used to force all the nodes to share the same decision variables of each task. With alternating direction method of multipliers (ADMoM) \cite{boyd2011distributed}, the centralized problem can be solved in a fully distributed way. Each task at a node shares its decision variables with the same task in the neighboring nodes and other tasks in the same node. As a result, the classification accuracy of each task in each node can be improved without sharing local and private data.
\begin{figure}[]
\centering
\includegraphics[width=0.4\textwidth]{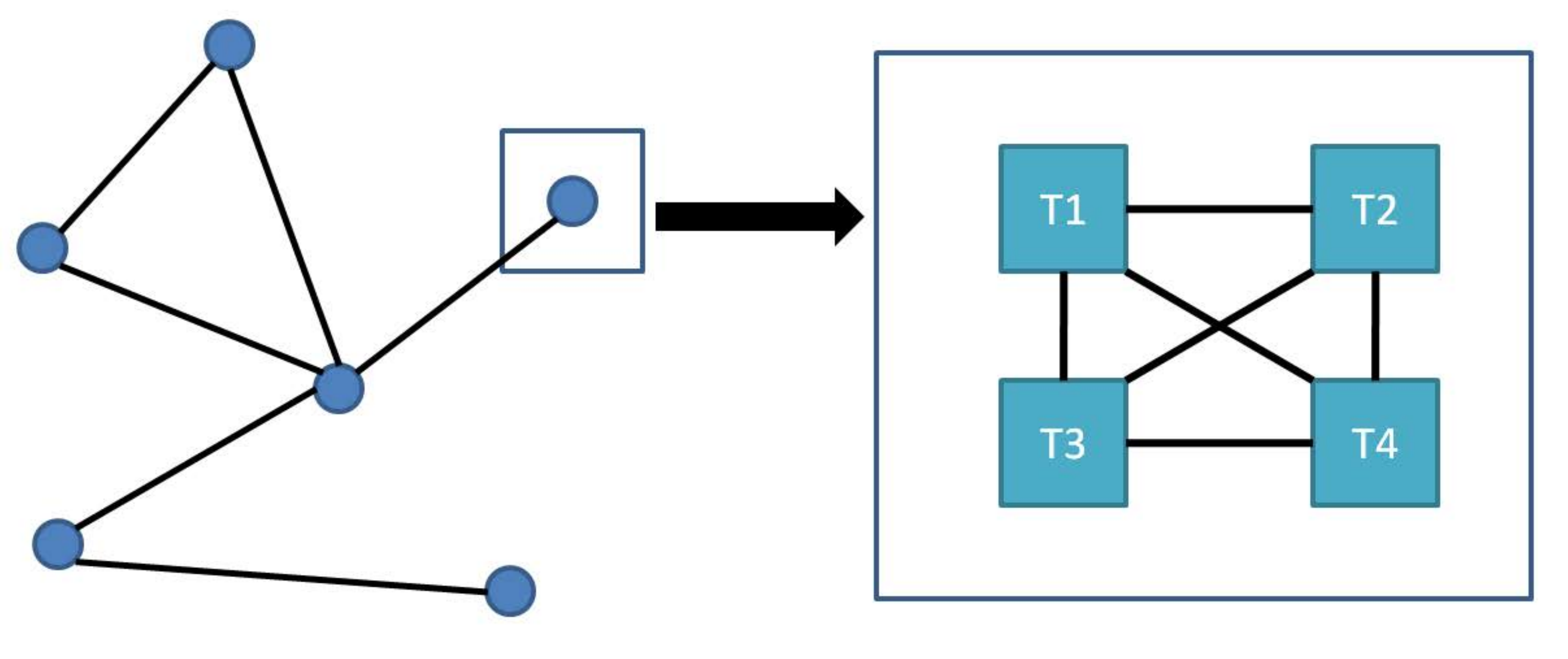}
\caption{Distributed transfer learning example. The left figure shows a network with $6$ nodes. The right figure shows that each node contains four tasks, which are trained  together in the network. }
\label{fig:DTSVM}
\end{figure}

The consensus-based distributed framework provides a way to address distributed transfer learning problems in connected networks. Since each task at a node makes decisions using its local data, the training process becomes more efficient and scalable. Allowing tasks and nodes to communicate their decision variables with others, we can achieve more accurate classifications without sharing private data between different tasks and different nodes, which effectively reduces the communication overhead and maintains privacy at the same time. Note that the problem of transfer learning between tasks in one node can be viewed as a transfer learning problem studied in \cite{evgeniou2004regularized}. Besides, the problem of distributed machine learning with a single task is a distributed support vector machines (DSVM) problem recently studied in \cite{forero2010consensus}. 

The proposed framework is a generalization of both centralized transfer learning scheme and distributed machine learning. It provides a large-scale transfer learning framework where each task transfers knowledge to other tasks and each node transfers knowledge to his neighboring nodes. Performances of all the tasks in each node
are illustrated in terms of their training efficiency and data privacy.

The rest of this paper is organized as follows. Section 2 presents a consensus-based centralized transfer learning approach on SVMs. Section 3 outlines the extended distributed transfer support vector machines (DTSVM). Section 4 and 5 present numerical results and concluding remarks, respectively.

{\noindent\bf Notations.} Boldface letters represent matrices (column vectors); $(\cdot)^T$ denotes matrix and vector transposition; $\parallel \cdot \parallel$ denotes the norm of the matrix or vector; $diag(\mathbf{y})$ denotes the diagonal matrix with $\mathbf{y}$ on its main diagonal; $\mathcal{V}$ denotes the set of nodes in a network; $\mathcal{B}_v$ denotes the set of neighboring nodes of node $v$; $\mathcal{T}$ denotes the set of tasks.
\section{Centralized Transfer Learning}
In this section, we present a centralized transfer learning approach on SVMs. Consider $T$ learning tasks with $\mathcal{T} =  \{ 1,...,T\}$ denotes the set of tasks. We assume that each task $t$ has a labeled training set $\mathcal{D}_t = \{ (\mathbf{x}_{tn},y_{tn})   | \mathbf{x}_{tn} \in \mathcal{X}_t, y_{tn} \in\{-1,+1\}  \}_{n=1}^{N_t}$, where $\mathcal{X}_t \subseteq \mathbb{R}^{p}$ represents the input space of task $t$. Note that $\mathcal{X}_t$ is different for each task, but has the same dimension $p$. For each task, a linear SVM aims to find a maximum-margin discriminant function $g_t(\mathbf{x}_t)=sign\left( \mathbf{x}_t^T \widehat{\mathbf{w}}_t^* +\widehat{b}_t ^*  \right)$, which gives input testing data $\mathbf{x}_t$ a label $-1$ or $+1$. Decision variables $\{\widehat{\mathbf{w}}_t^*,\widehat{b}_t^*\}$ can be found by solving the following minimization problem \cite{vapnik2013nature}:
\begin{equation}
\label{eq:SingleTaskSVM}
\begin{array}{c}
\min\limits_{\widehat{\mathbf{w}}_t,\widehat{b}_t, \{\xi_{tn} \}}  \frac{1}{2}\parallel \widehat{\mathbf{w}}_t \parallel_2^2 + C\sum\limits_{n=1}^{N_t} \xi_{tn} \\
\begin{array}{cc}
{\begin{array}{c}
\text{s.t.}\\ \ 
\end{array}}&{\begin{array}{c}
y_{tn}(\widehat{\mathbf{w}}_t^T\mathbf{x}_{tn}+\widehat{b}_t) \geq 1 -\xi_{tn};\\
\xi_{tn} \geq 0.
\end{array}}
\end{array}
\end{array}
\end{equation}
Note that, $\xi_{tn}$ is the slack variable, which accounts for non-separable case. Problem (\ref{eq:SingleTaskSVM}) is a traditional SVM problem for single task learning. With the assumption that different tasks are related to each other on the basis of similarity between distributions of samples $\mathcal{X}_t$ \cite{ben2003exploiting}, the decision variables $\widehat{\mathbf{w}}_t,\widehat{b}_t$ can be divided into: $\widehat{\mathbf{w}}_t = \mathbf{w}_{0} + \mathbf{w}_t;\widehat{b}_t = b_{0} + b_t$, where $\mathbf{w}_0$ and $b_0$ are common terms over all tasks, while $\mathbf{w}_t$ and $b_t$ are task specific terms \cite{evgeniou2004regularized,pan2010survey}. We further write the decision variables as:
\begin{equation}
\label{eq:DecisionVariablesTransfer}
\widehat{\mathbf{w}}_t = \mathbf{w}_{0t} + \mathbf{w}_t;\widehat{b}_t = b_{0t} + b_t,
\end{equation}
with $\mathbf{w}_{01} = ... = \mathbf{w}_{0T}$ and $b_{01} = ...= b_{0T}$ forcing all common terms to agree with each other among all tasks. Thus,
a consensus-based centralized approach of multi-task transfer learning can be formulated as the following problem: 
\begin{equation}
\label{eq:CentralizedTransfer}
\begin{array}{l}
\min\limits_{\{{\mathbf{w}}_{0t},{b}_{0t},{\mathbf{w}}_t,{b}_t, \{\xi_{tn} \}\}}  \frac{\epsilon_1}{2} \sum\limits_{t\in\mathcal{T}}\parallel  {\mathbf{w}}_{0t} \parallel_2^2 \\ \ \ \ \ \ \ \ \ \ \ \ \ \ \ \ \ \ \ + \frac{\epsilon_2}{2} \sum\limits_{t\in\mathcal{T}}\parallel {\mathbf{w}}_t\parallel_2^2  + TC \sum\limits_{t\in\mathcal{T}}\sum\limits_{n=1}^{N_t} \xi_{tn} \\\text{s.t.} \\
\begin{array}{ccc}
{\begin{array}{c}
y_{tn}(\widehat{\mathbf{w}}_{t}^T\mathbf{x}_{tn}+\widehat{b}_{t}) \geq 1 -\xi_{tn},\\
\xi_{tn} \geq 0,\\
\mathbf{w}_{0t} = \mathbf{w}_{0s},b_{0t}=b_{0s},
\end{array}}&{\begin{array}{c}
\forall t\in\mathcal{T};\\
\forall t\in\mathcal{T};\\
\forall t,s\in\mathcal{T},s\neq t.
\end{array}}&{\begin{array}{c}
(\ref{eq:CentralizedTransfer}a)\\
(\ref{eq:CentralizedTransfer}b)\\
(\ref{eq:CentralizedTransfer}c)
\end{array}}
\end{array}
\end{array}
\end{equation}
Note that, consensus constraints (\ref{eq:CentralizedTransfer}c) are used to restrict common terms. $\epsilon_1$ and $\epsilon_2$ are positive regularization parameters, which determine how much $\widehat{\mathbf{w}}_t$ differs in each task by controlling the size of $\mathbf{w}_{0t}$ and $\mathbf{w}_{t}$. When $\frac{\epsilon_1}{\epsilon_2}$ is large, $\mathbf{w}_{0t}$ tends to be equal to $0$, which makes all tasks unrelated. On the other hand, when $\frac{\epsilon_1}{\epsilon_2}$ is small, $\mathbf{w}_{t}$ tends to be equal to $0$, which makes all tasks find the same classifier. 

By solving Problem (\ref{eq:CentralizedTransfer}), we can find the decision variables $\widehat{\mathbf{w}}_t^*$ and $\widehat{b}_t^*$ simultaneously with information transferred through consensus constraints (\ref{eq:CentralizedTransfer}c) and common terms $\mathbf{w}_{0t}$ and $b_{0t}$. Problem (\ref{eq:CentralizedTransfer}) provides a centralized framework to transfer learning. In the following section, we further extend it to a distributed network.
\section{Distributed Transfer Learning}
Consider a network with $\mathcal{V}= \{1,..,V \}$ representing the set of nodes. Node $v\in\mathcal{V}$ only communicates with his neighboring nodes $\mathcal{B}_v \subseteq \mathcal{V}$. Without loss of generality, we assume that any two nodes in this network are connected by a path, i.e., there is no isolated node in this network. At each node $v$, $T$ labeled training sets $\mathcal{D}_{vt} =\left\lbrace (\mathbf{x}_{vtn},y_{vtn}) | \mathbf{x}_{vtn} \in \mathcal{X}_t, y_{vtn} \in\{-1,+1\}\right\rbrace_{n = 1}^{N_{vt}}$ of size $N_{vt}$ are available for each task $t\in\mathcal{T}$ (e.g., see Fig. \ref{fig:DTSVM}). 

The maximum-margin linear discriminant function at every node $v\in\mathcal{V}$ for each task $t\in\mathcal{T}$ can be described as $g_{vt}(\mathbf{x}_t) = \mathbf{x}_t^T\widehat{\mathbf{w}}_{vt}^* + \widehat{b}_{vt}^*$, where decision variables $\widehat{\mathbf{w}}_{vt}^* = \mathbf{w}_{0vt}^* + \mathbf{w}_{vt}^*$ and $\widehat{b}_{vt}^* = b_{0vt}^*+b_{vt}^*$. Note that there are two sets of consensus constraints, $\mathbf{w}_{011}=...=\mathbf{w}_{01T}=...=\mathbf{w}_{0V1}=...=\mathbf{w}_{0VT}$ and $b_{011}=...=b_{01T}=...=b_{0V1}=...=b_{0VT}$ are used to force all common terms of decision variables to agree with each other among all the nodes and all the tasks, while $\mathbf{w}_{1t}=...=\mathbf{w}_{Vt}$ and $b_{1t}=...=b_{Vt}$ are used to forcing all decision variables $\{\widehat{\mathbf{w}}_{vt},\widehat{b}_{vt} \}_{v\in\mathcal{ V}} $ of task $t$ to agree with each other among all the nodes. This approach enables each task $t$ at each node $v$ to classify any new input $\mathbf{x}_t$ to one of the two classes $\{+1,-1 \}$ without communicating $\mathcal{D}_{vt}$ to other nodes $v' \neq v$. The discriminant function $g_{vt}(\mathbf{x}_t)$ can be obtained by solving the following optimization problem: 
\begin{equation}
\label{eq:DistributedTransfer}
\begin{array}{l}
\min\limits_{\{{\mathbf{w}}_{0vt},{b}_{0vt},{\mathbf{w}}_{vt},{b}_{vt}, \{\xi_{vtn} \}\}}  \frac{\epsilon_1}{2} \sum\limits_{v\in\mathcal{V}} \sum\limits_{t\in\mathcal{T}}\parallel  {\mathbf{w}}_{0vt} \parallel_2^2  \\ \ \ \ \ \ \ \ \ \ \ \ \ \ \ \  + \frac{\epsilon_2}{2} \sum\limits_{v\in\mathcal{V}} \sum\limits_{t\in\mathcal{T}}\parallel {\mathbf{w}}_{vt}\parallel_2^2  + VTC \sum\limits_{v\in\mathcal{V}}  \sum\limits_{t\in\mathcal{T}}\sum\limits_{n=1}^{N_t} \xi_{vtn} \\
\text{s.t.}\\
\begin{array}{ll}
{\begin{array}{c}
y_{vtn}(\widehat{\mathbf{w}}_{vt}^T\mathbf{x}_{vtn}+\widehat{b}_{vt}) \geq 1 -\xi_{vtn},\\
\xi_{vtn} \geq 0,\\
\mathbf{w}_{0vt} = \mathbf{w}_{0vs},b_{0vt}=b_{0vs},\\
{\mathbf{w}}_{0vt} = {\mathbf{w}}_{0ut},{b}_{0vt}={b}_{0ut},\\
{\mathbf{w}}_{vt} = {\mathbf{w}}_{ut},{b}_{vt}={b}_{ut},
\end{array}}&{\begin{array}{c}
\forall v\in\mathcal{V},t\in\mathcal{T};\\
\forall v\in\mathcal{V},t\in\mathcal{T};\\
\forall v\in\mathcal{V},t,s\in\mathcal{T},s\neq t;\\
\forall v\in\mathcal{V},t\in\mathcal{T},u\in\mathcal{B}_v;\\
\forall v\in\mathcal{V},t\in\mathcal{T},u\in\mathcal{B}_v.
\end{array}}
\end{array}
\end{array}
\end{equation}
In the above problem, the third and the fourth constraints impose the consensus on the common terms $\mathbf{w}_{0vt}$ and $b_{0vt}$ at every node $v$ for each task $t$, while the fourth and the fifth constraints impose the consensus on decision variables $\widehat{\mathbf{w}}_{vt}:=\mathbf{w}_{0vt}+\mathbf{w}_{vt}$ and $\widehat{b}_{vt}:=b_{0vt}+b_{vt}$ across neighboring nodes for each task $t$. 

To solve Problem (\ref{eq:DistributedTransfer}), we first define the vector of decision variables $\mathbf{r}_v:=[{\bf{w}}_{0vt}^T,b_{0vt},\mathbf{w}_{vt}^T,b_{vt}]^T$, the augmented matrix $\mathbf{X}_{vt}:=[(\mathbf{x}_{vt1},...,\mathbf{x}_{vtN_v})^T,\mathbf{1}_{vt}]$, the diagonal label matrix $\mathbf{Y}_{vt}:=diag([y_{vt1},...,y_{vtN_{vt}}])$, and the vector of slack variables $\mathbf{\xi}_{vt}:=[\xi_{vt1},...,\xi_{vtN_{vt}}]^T$. With these definitions, it follows readily that ${\bf{w}}_{0vt}=[\hat{\mathbf{I}},\mathbf{0}]\mathbf{r}_v$ and $\widehat{\mathbf{w}}_{vt}=[\hat{\mathbf{I}},\hat{\mathbf{I}}]\mathbf{r}_v$ where $[\hat{\mathbf{I}},\mathbf{0}]:=[\hat{\mathbf{I}}_{p+1},\mathbf{0}_{p+1}]$ and $[\hat{\mathbf{I}},\hat{\mathbf{I}}]:=[\hat{\mathbf{I}}_{p+1},\hat{\mathbf{I}}_{p+1}]$. $\hat{\mathbf{I}}_{p+1}$ is a $(p+1)\times(p+1)$ identity matrix with its $(p+1,p+1)$-st entry being $0$. Thus, Problem (\ref{eq:DistributedTransfer}) can be rewritten as

\begin{equation}
\label{eq:DistributedTransferMatrix}
\begin{array}{l}
\min\limits_{ \{  \mathbf{r}_{vt},\xi_{vt},\varphi_{vts},\omega_{vut} \}  } \frac{\epsilon_1}{2} \sum\limits_{v\in\mathcal{V}}\sum\limits_{t\in\mathcal{T}} \mathbf{r}_{vt}^T {\mathbf{M}_1}\mathbf{r}_{vt}  \\ \ \ \ \ \ \ \ \ \ \ \ \ \ \  +\frac{\epsilon_2}{2} \sum\limits_{v\in\mathcal{V}}\sum\limits_{t\in\mathcal{T}} \mathbf{r}_{vt}^T {\mathbf{M}_2} \mathbf{r}_{vt}  + VTC  \sum\limits_{v\in\mathcal{V}} \sum\limits_{t\in\mathcal{T}}\mathbf{\xi}_{vt} \\ \text{s.t.}\\
\begin{array}{cc}
{\begin{array}{c}
\mathbf{Y}_{vt}\mathbf{X}_{vt}[\mathbf{I,I}]\mathbf{r}_{vt} \succeq \mathbf{1}_{vt} - \mathbf{\xi}_{vt},\\
\mathbf{\xi}_{vt} \succeq \mathbf{0}_{vt},\\
\left[ \mathbf{I,0}\right] \mathbf{r}_{vt} = \varphi_{vts},\varphi_{vts}=\left[ \mathbf{I,0}\right]\mathbf{r}_{vs},\\
\mathbf{r}_{vt} = \omega_{vut},\omega_{vut} = \mathbf{r}_{ut},
\end{array}}&{\begin{array}{c}
v\in \mathcal{V},t\in\mathcal{T};\\
v\in \mathcal{V},t\in\mathcal{T};\\
v\in \mathcal{V},t,s\in\mathcal{T},s\neq t;\\
v\in \mathcal{V},t\in\mathcal{T},u\in\mathcal{B}_v,
\end{array} }
\end{array}
\end{array}
\end{equation}
where $\varphi_{vts}$ is used to decompose the common term $\left[ \mathbf{I,0}\right]\mathbf{r}_{vt}$ of task $t$ to other tasks $s \neq t$, and $\omega_{vut}$ is used to decompose the decision variable $\mathbf{r}_v$ at node $v$ to its neighboring nodes $u\in\mathcal{B}_v$. Note that $\left[ \mathbf{I,0}\right]: =\left[ \mathbf{I}_{p+1},\mathbf{0}_{p+1}\right] $, $\left[ \mathbf{I,I}\right]: =\left[ \mathbf{I}_{p+1},\mathbf{I}_{p+1}\right] $, $\mathbf{M}_1 = [ \hat{\mathbf{I}},\mathbf{0} ]^T[ \hat{\mathbf{I}},\mathbf{0}]$ and  $\mathbf{M}_2 = [ \mathbf{0},\hat{\mathbf{I}} ]^T[\mathbf{0} ,\hat{\mathbf{I}}]$. 

Problem (\ref{eq:DistributedTransferMatrix}) can be solved iteratively in a distributed way with ADMoM \cite{boyd2011distributed}, which is shown as the following proposition. 
\begin{proposition}
\label{Proposition1}
With $\alpha_{vt}^{(0)}=\mathbf{0}_{(p+1)\times 1}$ and $\beta_{vt}^{(0)}=\mathbf{0}_{(2p+2)\times 1}$, Problem (\ref{eq:DistributedTransferMatrix}) can be solved by the following iterations:
\begin{equation}
\label{eq:DistributedTransferSoli1}
\begin{array}{l}
\lambda_{vt}^{(k+1)} \in\arg\max\limits_{\mathbf{0}_{vt} \preceq \lambda_{vt} \preceq  VTC \mathbf{1}_{vt} }-\frac{1}{2}\lambda_{vt}^T\mathbf{Y}_{vt}\mathbf{X}_{vt}[\mathbf{I,I}]\mathbf{U}_{vt}^{-1}[\mathbf{I,I}]^T\mathbf{X}_{vt}^T\mathbf{Y}_{vt}\lambda_{vt} \\ \ \ \ \ \ \ \ \ \ \ \ \ \ \ \ \ + (\mathbf{1}_{vt}+\mathbf{Y}_{vt}\mathbf{X}_{vt}[\mathbf{I,I}]\mathbf{U}_{vt}^{-1}\mathbf{f}_{vt}^{(k)})^T\lambda_{vt},
\end{array}
\end{equation}
\begin{equation}
\label{eq:DistributedTransferSoli2}
\begin{array}{c}
\mathbf{r}_{vt}^{(k+1)}=\mathbf{U}_{vt}^{-1} \bigg( [\mathbf{I,I}]^T\mathbf{X}_{vt}^T\mathbf{Y}_{vt}\lambda_{vt}^{(k+1)} - \mathbf{f}_{vt}^{(k)} \bigg),
\end{array}
\end{equation}
\vspace*{0.005cm} 
\begin{equation}
\label{eq:DistributedTransferSoli3}
\alpha_{vt}^{(k+1)} = \alpha_{vt}^{(k)} + \frac{\eta_1}{2}[\mathbf{I,0}]\sum\limits_{s\in\mathcal{T},s\neq t}(\mathbf{r}_{vt}^{(k+1)}-\mathbf{r}_{vs}^{(k+1)}),
\end{equation}
\begin{equation}
\label{eq:DistributedTransferSoli4}
\beta_{vt}^{(k+1)} = \beta_{vt}^{(k)} + \frac{\eta_2}{2}\sum\limits_{u\in\mathcal{B}_v}(\mathbf{r}_{vt}^{(k+1)}-\mathbf{r}_{ut}^{(k+1)}),
\end{equation}
where
\begin{equation}
\label{eq:DistributedTransferU}
\mathbf{U}_{vt} = \epsilon_1\mathbf{M}_1 + \epsilon_2 \mathbf{M}_2 + 2\eta_1(T-1)[\mathbf{I,0}]^T[\mathbf{I,0}]+2\eta_2|\mathcal{B}_v|\mathbf{I}_{2p+2}, 
 \end{equation}
and 
\begin{equation}
\label{eq:DistributedTransferf}
\begin{array}{l}
\mathbf{f}_{vt}^{(k)} =  2[\mathbf{I,0}]^T\alpha_{vt}^{(k)}  + 2\beta_{vt}^{(k)}\\ \ \ \ \ \ -\eta_1 \sum\limits_{s\in\mathcal{T},s\neq t}[\mathbf{I,0}]^T[\mathbf{I,0}](\mathbf{r}_{vt}^{(k)}+\mathbf{r}_{vs}^{(k)}) - \eta_2 \sum\limits_{u\in\mathcal{B}_v} (\mathbf{r}_{vt}^{(k)}+\mathbf{r}_{ut}^{(k)}).
\end{array}
\end{equation}
\end{proposition}
\begin{proof}
See Appendix A.
\end{proof}
In Proposition 1, each task at node $v$ computes $\lambda_{vt}$ by (\ref{eq:DistributedTransferSoli1}), then it computes $\mathbf{r}_{vt}$ by (\ref{eq:DistributedTransferSoli2}) using the new $\lambda_{vt}$. In the next step, each task at node $v$ sends $\mathbf{r}_{vt}$ to all the other tasks $s\in\mathcal{T},s\neq t$, and each node of task $t$ broadcasts $\mathbf{r}_{vt}$ to the neighboring nodes $u\in\mathcal{B}_v$. Then, $\alpha_{vt}$ updates by (\ref{eq:DistributedTransferSoli3}) with $\mathbf{r}_{vs}$ from the other tasks $s\in\mathcal{T},s\neq t$, while $\beta_{vt}$ updates by (\ref{eq:DistributedTransferSoli4}) with $\mathbf{r}_{ut}$ from neighboring nodes $u\in\mathcal{B}_v$. Then, each task at node $v$ repeats computing $\lambda_{vt}$ by (\ref{eq:DistributedTransferSoli1}) with $\alpha_{vt}$ and $\beta_{vt}$, and the iteration goes until convergence. Note that, at each iteration $k$, each task at each node can evaluate its own discriminant function for any input data $\mathbf{x}_t$ as:
\begin{equation}
\label{eq:DistributedTransferDiscriminantFunction}
g_{vt}(\mathbf{x}_t)= [\mathbf{x}_t^T,1][\mathbf{I},\mathbf{I}]\mathbf{r}_{vt}.
\end{equation}
Proposition 1 illustrates the iterations of distributed transfer support vector machines (DTSVM). It is a fully distributed algorithm which does not require a fusion center to store or process all the data. Each iteration requires calculating $\lambda_{vt}$, $\mathbf{r}_{vt}$, $\alpha_{vt}$ and $\beta_{vt}$. The computation of $\lambda_{vt}$ is quadratic programming that can be solved in polynomial time.  $\mathbf{r}_{vt}$, $\alpha_{vt}$ and $\beta_{vt}$ can be calculated directly. It can be easily shown that the inverse of $\mathbf{U}_{vt}$ always exists. The information transferred between nodes is the decision variables $\mathbf{r}_{vt}$. This scheme maintains the privacy of sensitive data and reduces the communication overhead at the same time since the data is kept at each node. Our DTSVM algorithm also has no assumptions on the form of data and networks, and thus, it can be used in various situations. Moreover, since decision variables $\mathbf{r}_{vt}$ are updated at each iteration, adding or deleting nodes and modifying connections do not require rerunning of the whole algorithm. In addition, the proof of the convergence of the iterations to the solution of Problem (\ref{eq:DistributedTransferMatrix}) is provided at the end of Appendix A.
\section{Numerical Experiments}
In this section, we present numerical experiments of DTSVM. We use the MNIST database of handwritten digits to evaluate the distributed transfer learning algorithm \cite{MNIST}. The MNIST database contains images of digit ``$0$" to ``$9$", here we set classifying ``$3$" and ``$6$" as Task 1, classifying ``$5$" and ``$4$" as Task 2 and classifying ``$8$" and ``$9$" as Task 3. Note that, Task 1 and 2 are the target tasks that we aim to decrease their classification risks, while Task 3 is the source task that helps us to achieve that. All the images have been pre-processed with principal component analysis (PCA) into vectors with a dimension of $10$ \cite{jolliffe2002principal}. We further define the degree of a node $v\in\mathcal{V}$ as the actual number of neighboring nodes $\mathcal{B}_v$ divided by the most achievable number of neighbors $|\mathcal{V}|-1$, and the degree of the network $\mathcal{V}$ as the average degree of all the nodes $v\in\mathcal{V}$. 

For comparison purposes, we also present the results of centralized support vector machines (CSVM) and distributed support vector machines (DSVM). The algorithm of CSVM can be acquired from \cite{vapnik2013nature}. The algorithm of DSVM can be found in \cite{forero2010consensus}, which only shares the values of decision variables during the training process. We will show later that the information from the nodes with DTSVM can also improve the performance of the nodes with DSVM. 

From Fig. \ref{fig:Ex0}, we can see that the classification risks of DTSVM are lower than the risks of both DSVM and CSVM, thus, transfer learning improves the performances of the tasks. Moreover, we can see that Task 1 benefits more than Task 3 as the risks of Task 1 in DTSVM decrease more. 

From Fig. \ref{fig:Ex1} and Fig. \ref{fig:Ex1C}, we can see that parameters $C$, $\varepsilon_1$ and $\varepsilon_2$ are related to the performance of transfer learning. $C$ indicates the trade-off between a larger margin and a smaller error penalty. Parameters $\varepsilon_1$ and $\varepsilon_2$ control the difference of decision variables between different tasks. When $\varepsilon_1/\varepsilon_2$ is large, $\mathbf{w}_{0t}$ tends to be $0$, i.e., all tasks tend to be not related, however, when $\varepsilon_1/\varepsilon_2$ is small, $\mathbf{w}_{t}$ tends to be $0$, i.e., all tasks tend to be same, both of the cases will decrease the classification accuracy. We can see from Fig. \ref{fig:Ex1} and Fig. \ref{fig:Ex1C} that the improvement of the performance requires a proper tuning of these parameters. 

Fig. \ref{fig:Ex3} shows the results when the training data of the target task, i.e., Task 1, is limited and has unbalanced labels. We can see that transfer learning can also improve the classification accuracy of these cases. Note that when there are only $2$ training samples of digit ``$3$" in Task 1, some nodes have only training samples of digit ``$6$", but the DTSVM can still find classifiers better than CSVM. 
\begin{figure}[]
\centering
\subfigure{\includegraphics[width=0.23\textwidth]{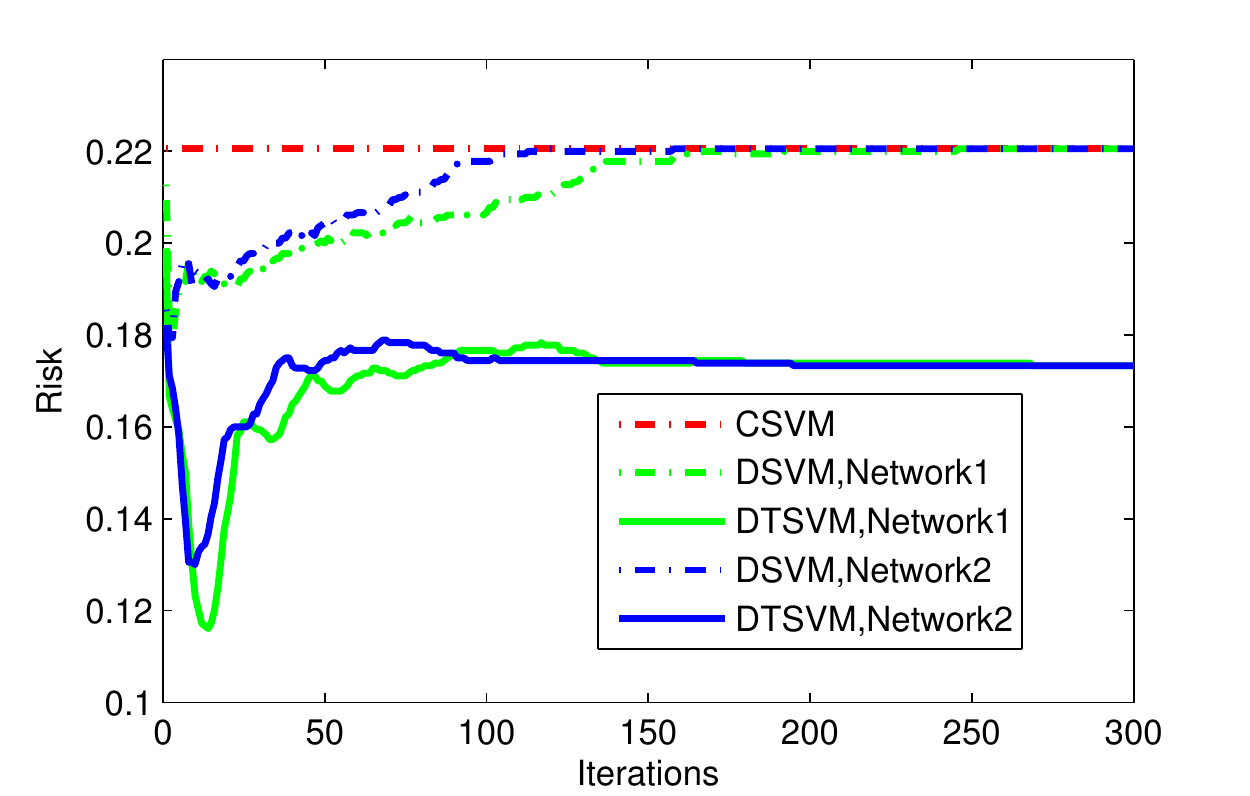}}
\subfigure{\includegraphics[width=0.23\textwidth]{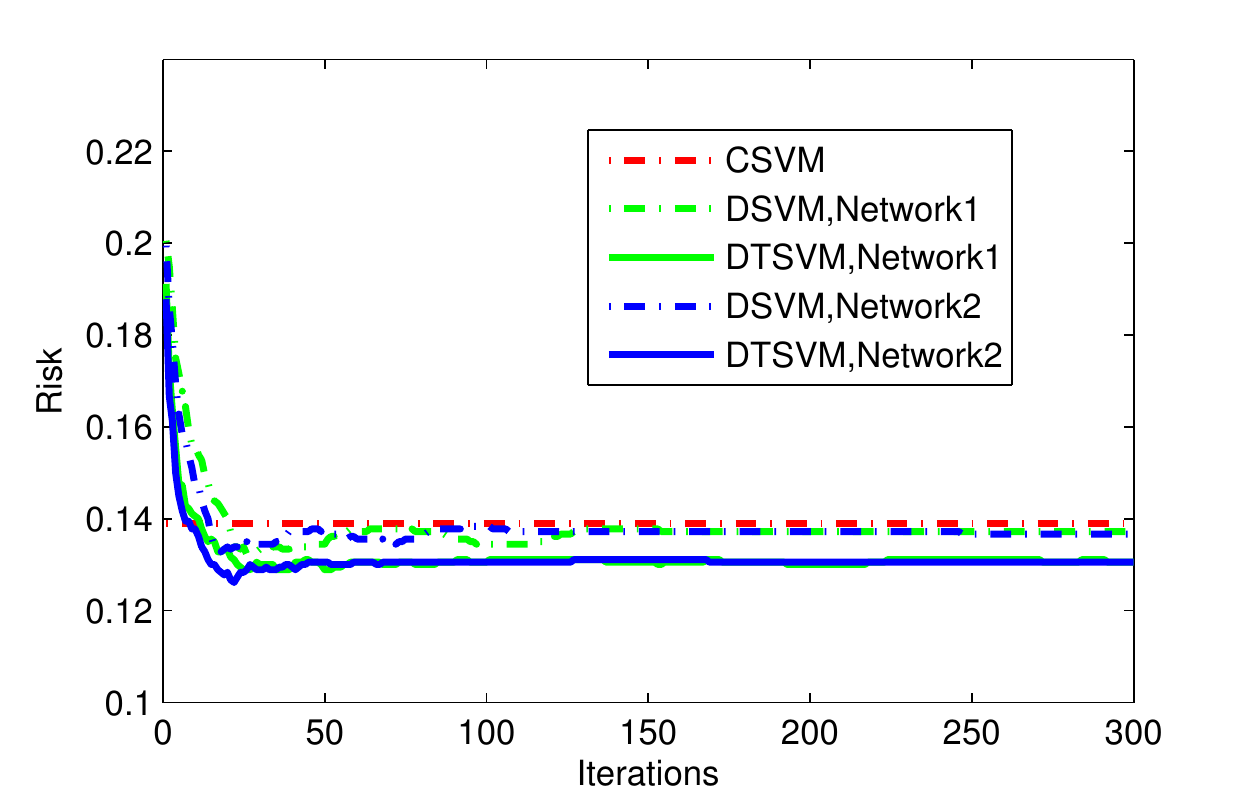}}
\vspace{-3mm}
\caption{Evolution of the global risks of DTSVM and DSVM \cite{forero2010consensus} training Task 1 and Task 3. The left figure and the right figure show the results of Task 1 and Task 3. Both Task 1 and Task 3 have $1800$ testing samples, and Task 3 has $800$ training samples, but Task 1 only has $200$ training samples. Note that both tasks are balanced. Network 1 has $20$ nodes with a degree of $0.6368$, while Network 2 has $10$ nodes with a degree of $0.8889$. Note that $C = 0.01$, $\epsilon_1 = 1$, $\epsilon_2 = 1$, $\eta_1 = 1$ and $\eta_2 = 1$. }
\label{fig:Ex0}
\end{figure}
\begin{figure}[]
\centering
\subfigure{\includegraphics[width=0.23\textwidth]{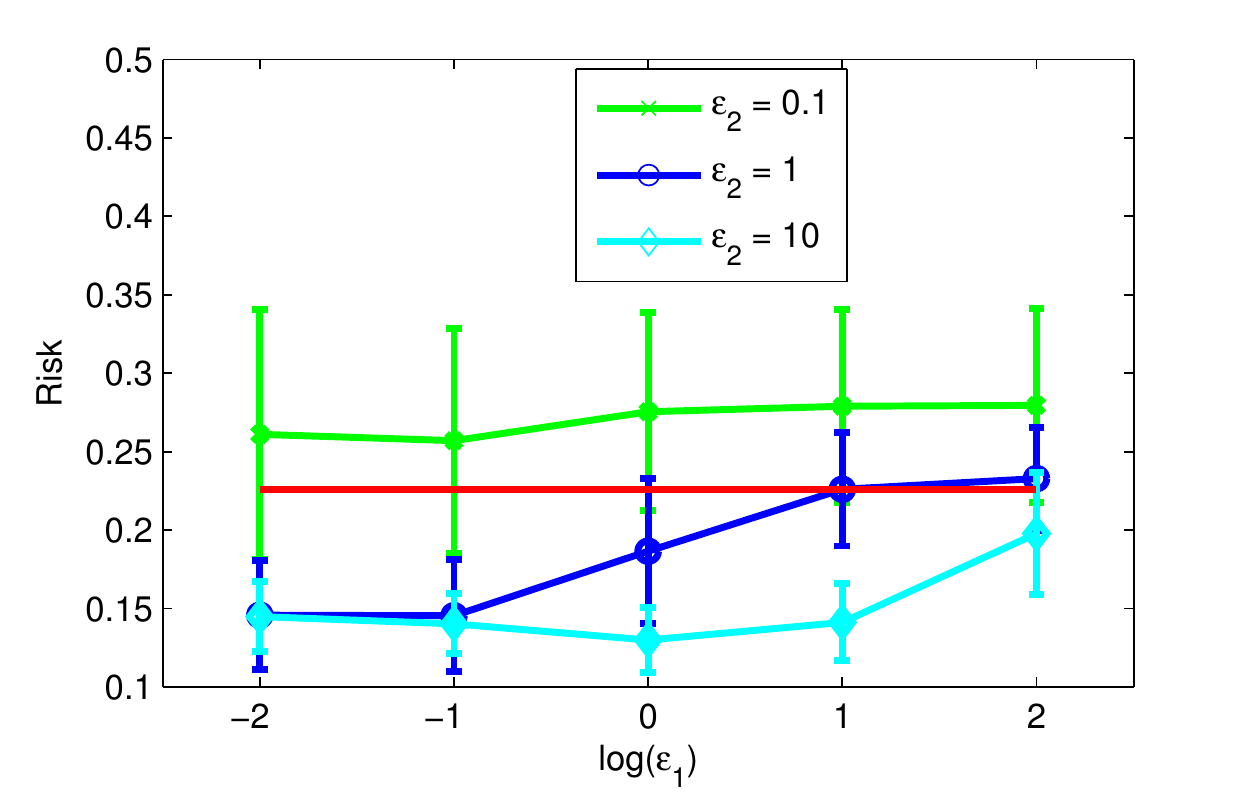}}
\subfigure{\includegraphics[width=0.23\textwidth]{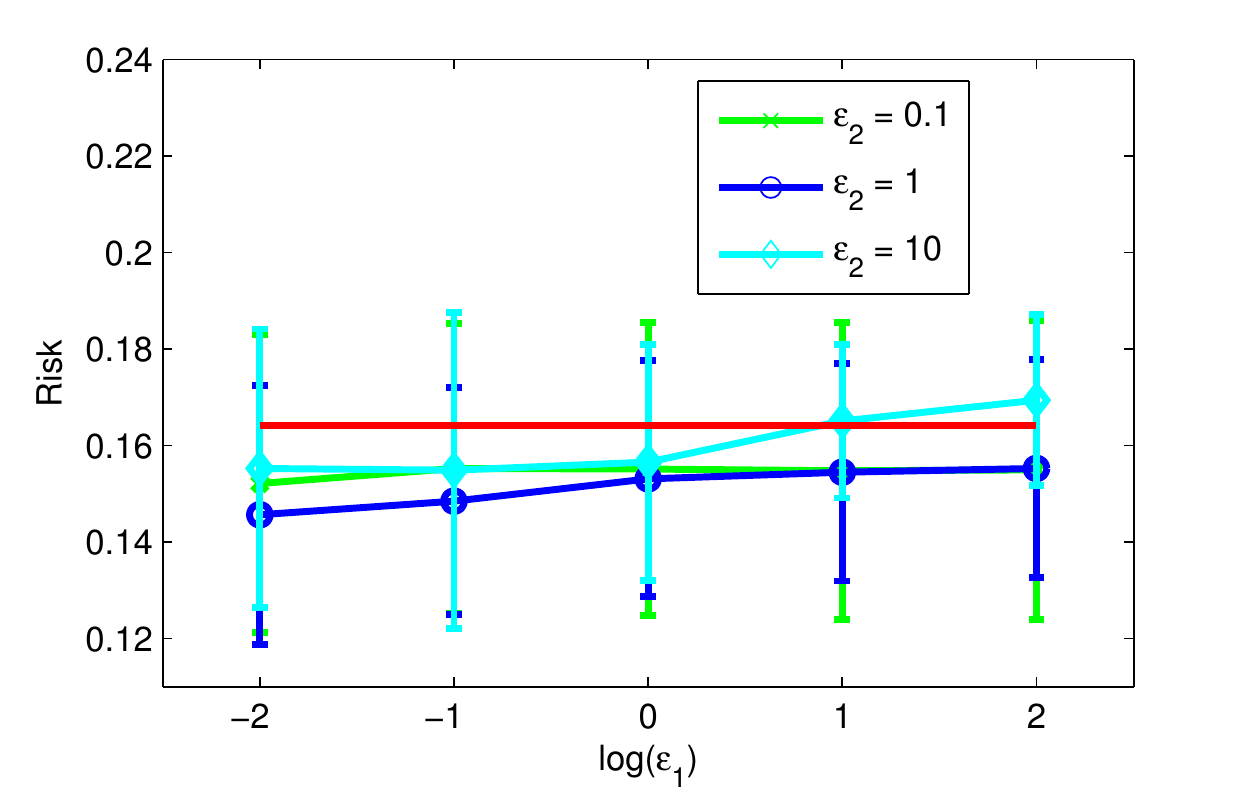}}
\vspace{-3mm}
\caption{Global convergent risks of DTSVM training Task 1 and Task 3 with different $\varepsilon_1$ and $\varepsilon_2$. The left figure and the right figure show the results of Task 1 and Task 3. Task 1 and Task 3 have $1800$ testing samples, Task 1 has $50$ training sample, Task 3 has $400$ training sample. The risks are calculated $15$ times with randomly selected samples. The red line shows the mean risks of CSVM. The network contains $10$ nodes with a degree of $0.8667$. Note that $C=0.01$, $\eta_1 = 1$ and $\eta_2 = 1$. }
\label{fig:Ex1}
\end{figure}
\begin{figure}[]
\centering
\subfigure{\includegraphics[width=0.23\textwidth]{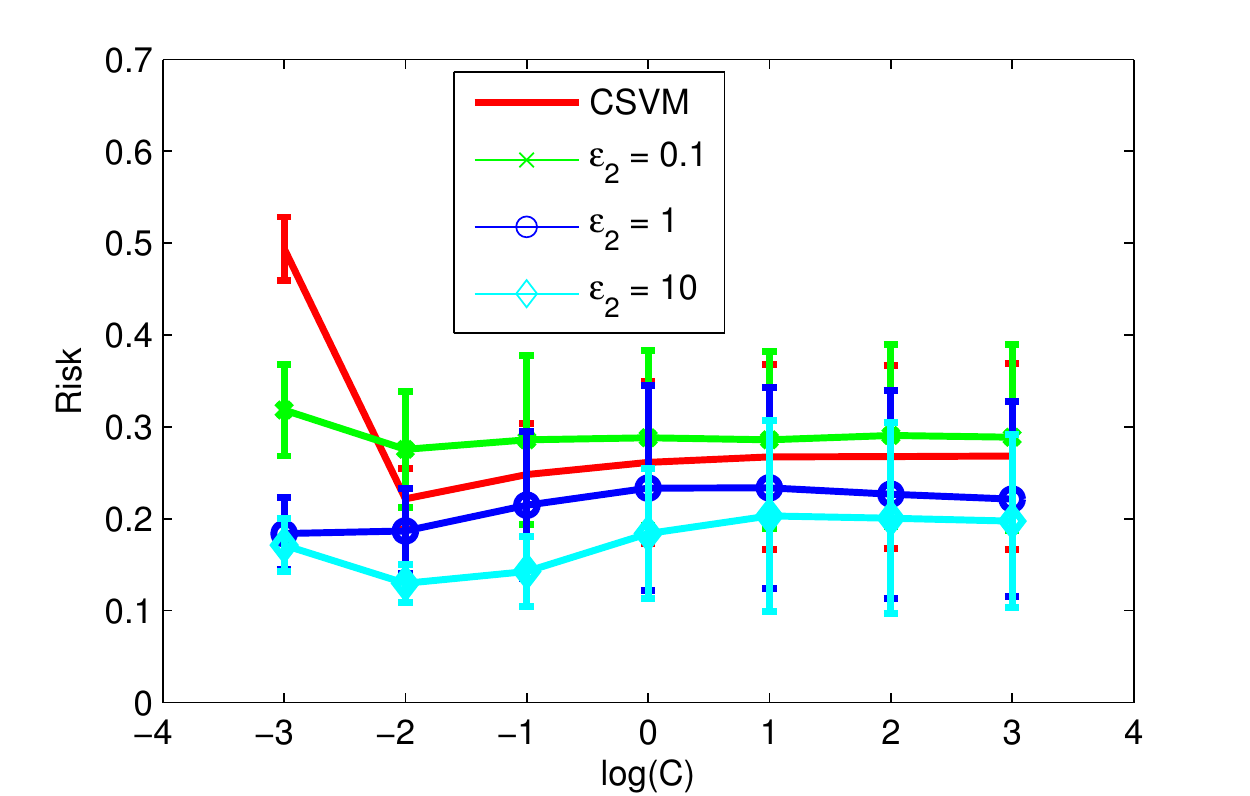}}
\subfigure{\includegraphics[width=0.23\textwidth]{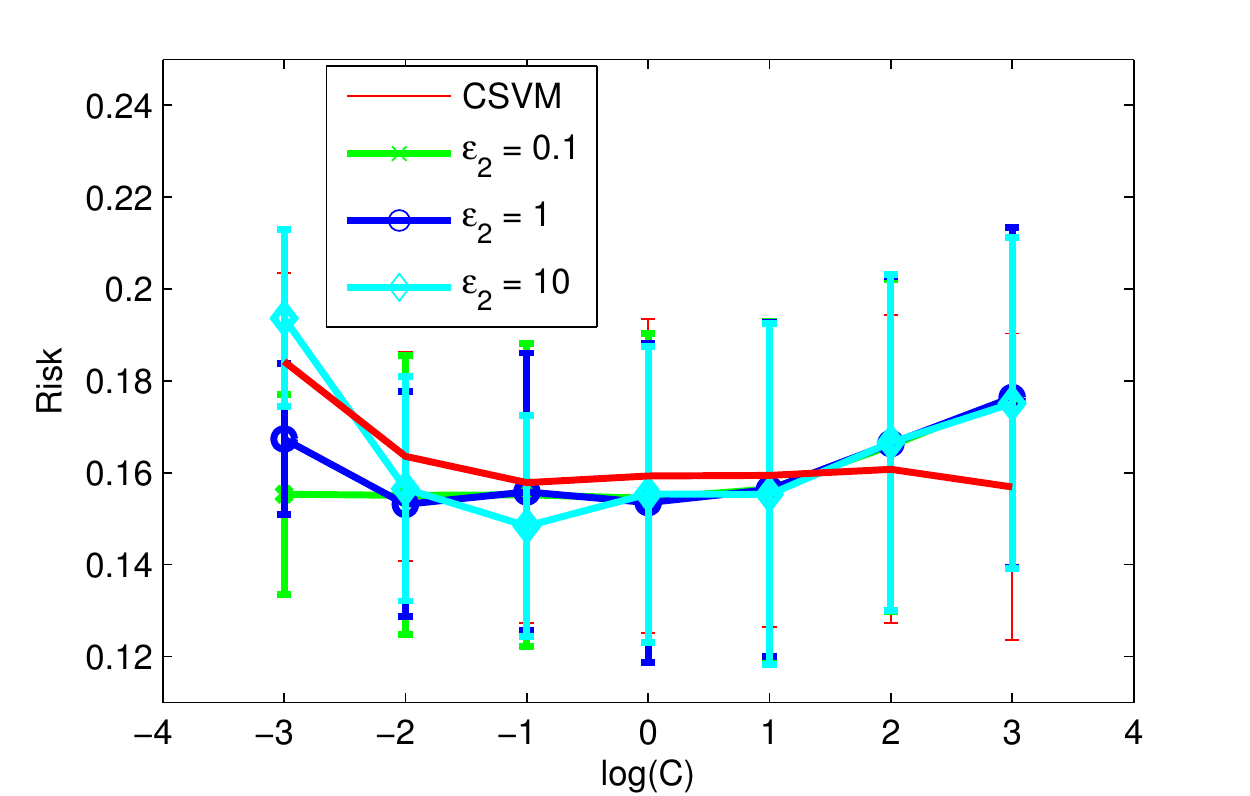}}
\vspace{-3mm}
\caption{Global convergent risks of DTSVM training Task 1 and Task 3 with different $C$ and $\varepsilon_2$. The left figure and the right figure show the results of Task 1 and Task 3. Task 1 and Task 3 have $1800$ testing samples, Task 1 has $50$ training sample, Task 3 has $400$ training sample. The risks are calculated $15$ times with randomly selected samples. The network contains $10$ nodes with a degree of $0.8667$. Note that $\varepsilon_1 = 1$, $\eta_1 = 1$ and $\eta_2 = 1$. }
\label{fig:Ex1C}
\end{figure}
\begin{figure}[]
\centering
\subfigure{\includegraphics[width=0.23\textwidth]{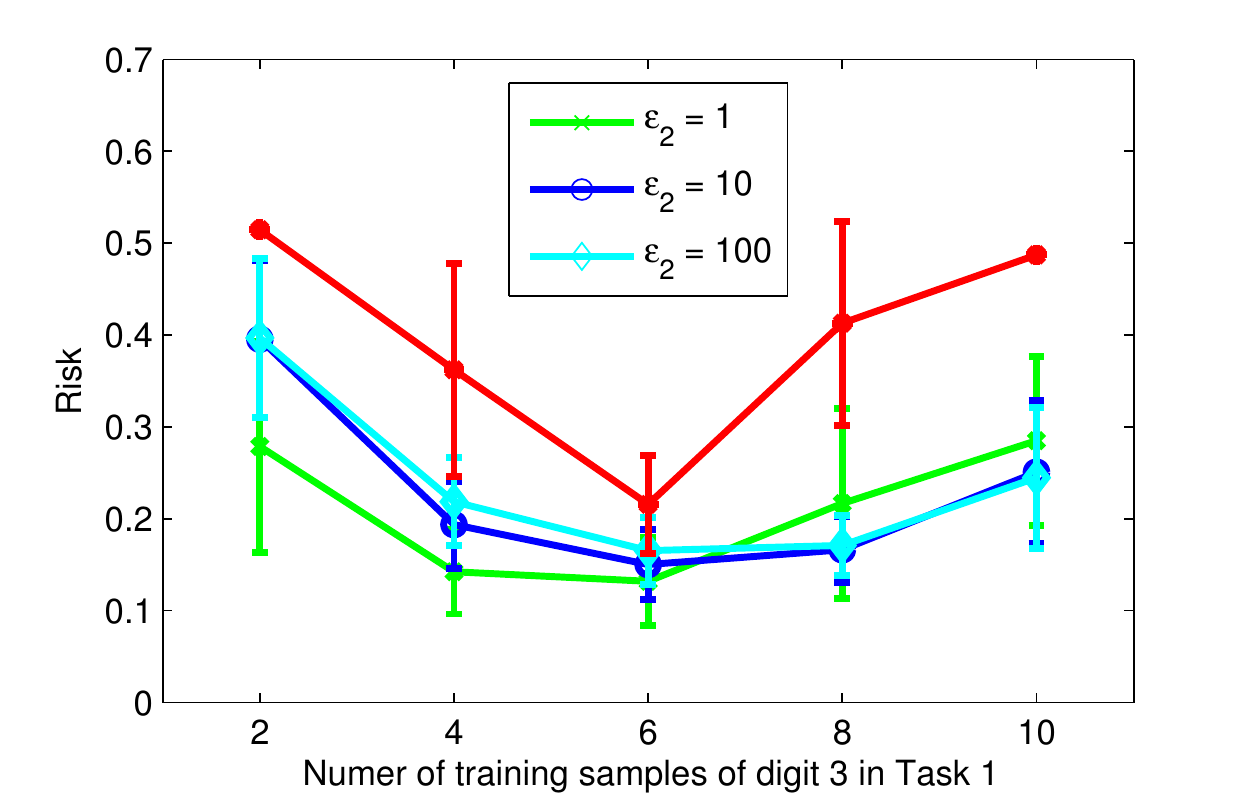}}
\subfigure{\includegraphics[width=0.23\textwidth]{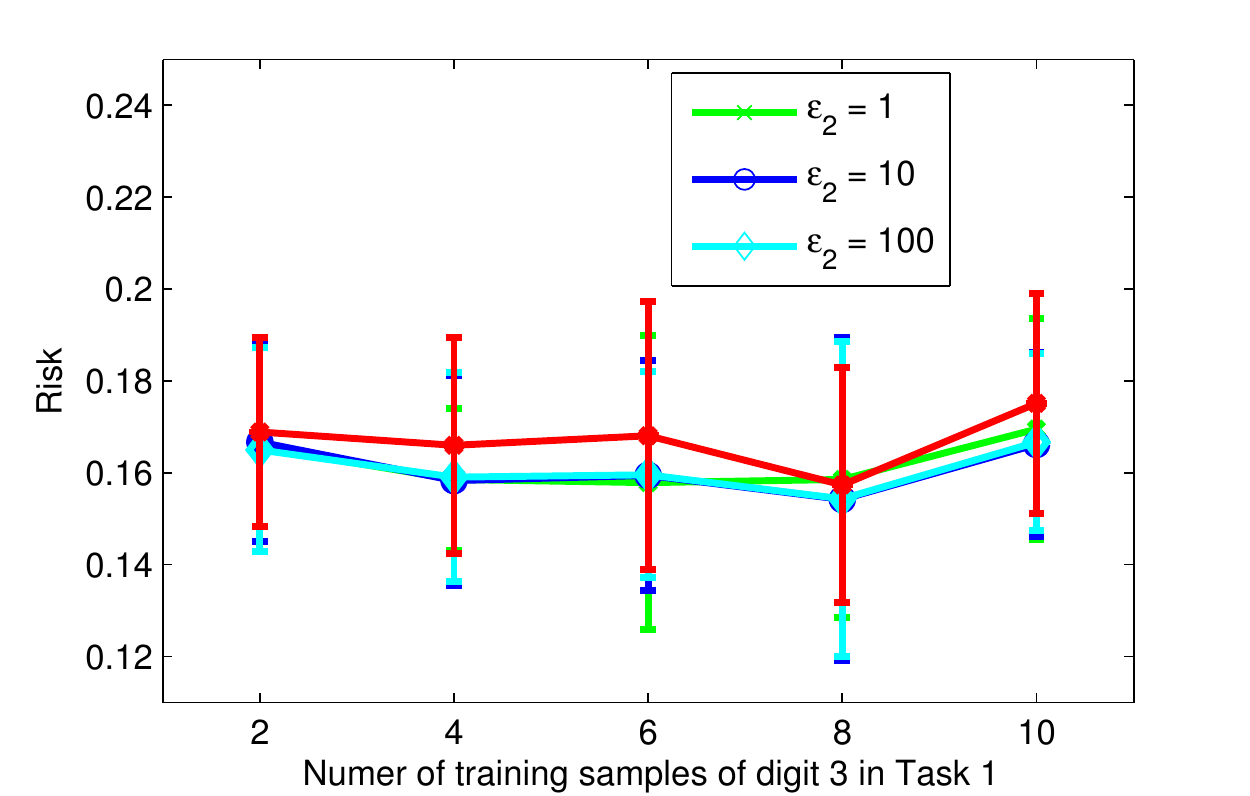}}
\vspace{-3mm}
\caption{Global convergent risks of DTSVM training Task 1 and Task 3 when Task 1 has $12$ training samples with unbalanced labels. The left figure and the right figure show the results of Task 1 and Task 3. Note that Task 3 has $200$ training samples with balanced labels. The risks are calculated $15$ times with randomly selected samples. The red line shows the risks of CSVM. The network is a fully connected network wth $4$ nodes. Note that $\varepsilon_1 = 1$, $\eta_1 = 1$ and $\eta_2 = 1$ and $C = 0.01$.}
\label{fig:Ex3}
\end{figure}
Fig. \ref{fig:DSVMandDTSVM} and Table \ref{tab:Ex4} show the results when the data is trained using DSVM and DTSVM together in the same network. Nodes who contain the data from the source task will train with DTSVM, while nodes who lack that will train with DSVM. We can see that nodes with DTSVM have lower risks. Moreover, nodes with DSVM also have lower risks as they receive information from nodes with DTSVM. This experiment shows that the performances of the nodes who lack training data from the source tasks can be improved with the knowledge transferred from the nodes who contain that data.
\begin{figure}[]
\centering
\subfigure{\includegraphics[width=0.23\textwidth]{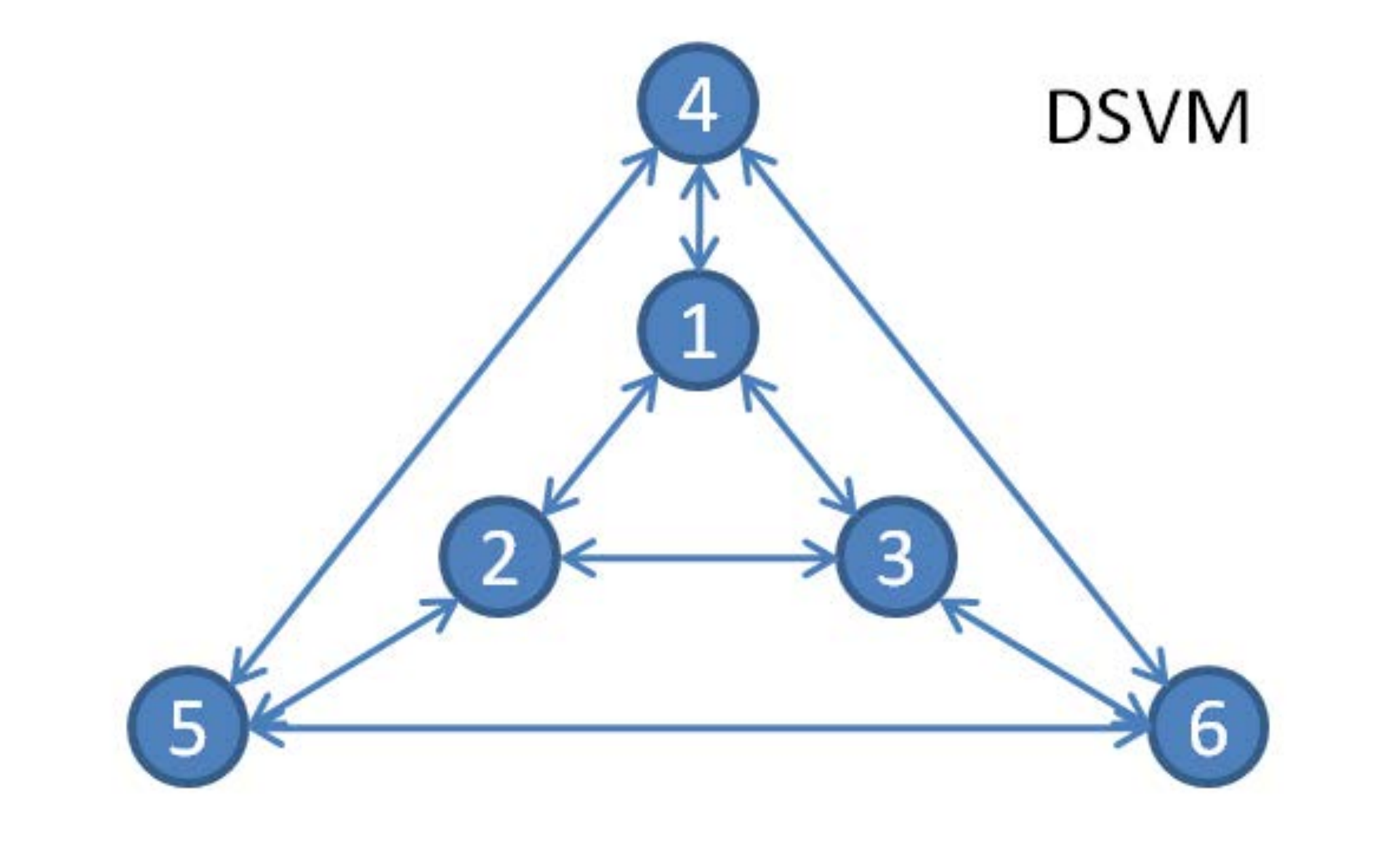}}
\subfigure{\includegraphics[width=0.23\textwidth]{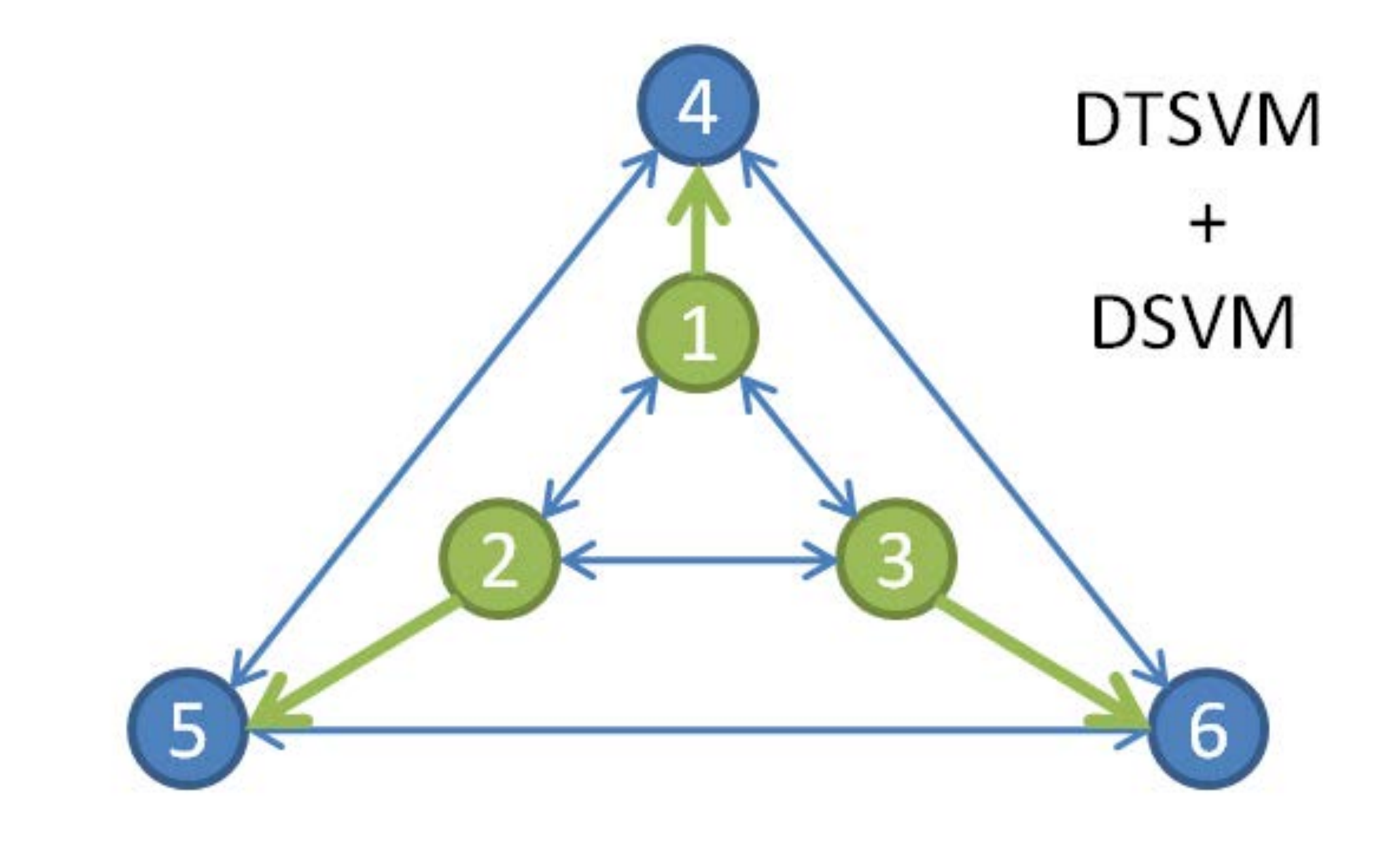}}
\vspace{-3mm}
\caption{DSVM and DTSVM train target task, i.e., Task 2 in the same network. The network has $6$ nodes, each node has $10$ training samples and $300$ testing samples from Task 2. Node 1, 2 and 3 contain $600$ training samples and $1800$ testing samples from the source task, i.e., Task 3. The left figure shows the case of training Task 2 with traditional DSVM, while the right figure shows the case when Node 1, 2 and 3 train Task 2 and 3 with DTSVM and Node 4, 5 and 6 train only Task 2 with DSVM, but Node 1, 2 and 3 also send their decision variables to Node 4, 5 and 6, respectively. Note that $\varepsilon_1 = 1$, $\varepsilon_2 = 10$, $\eta_1 = 1$ and $\eta_2 = 1$ and $C = 0.01$. Numerical results are shown in Table \ref{tab:Ex4}. The risks are calculated  $20$ times with randomly selected samples.}
\label{fig:DSVMandDTSVM}
\end{figure}
\begin{table}[]
\caption{Convergent classification risks $(\%)$ of Task 2. ``G'' indicates the global risks. ``Left'' and ``Right'' indicates the networks in Fig. \ref{fig:DSVMandDTSVM}.}
\vspace{-5mm}
\label{tab:Ex4}
\begin{center}
\begin{small}
\begin{sc}
\begin{tabular}{|c|c|c|c|c|c|c|c|}
\hline
Node & 1 & 2 & 3 & 4  & 5 & 6 & G \\
\hline
Left & 38.3 & 38.2 & 38.5 & 38.5 & 38.1 & 37.9 & 38.3 \\
\hline
STD & 5.7 & 5.5 & 6.3 & 5.1 & 6.1 & 4.1 & 4.9 \\
\hline
Right & 14.6 & 14.8 & 14.6 & 14.2 & 14.6 & 14.6 & 14.6 \\
\hline
STD & 2.8 & 2.5 & 2.7 & 2.9 & 2.7 & 2.4 & 1.9 \\
\hline
\end{tabular}
\end{sc}
\end{small}
\end{center} 
\end{table}
\begin{figure}[]
\centering
\includegraphics[width=0.5\textwidth]{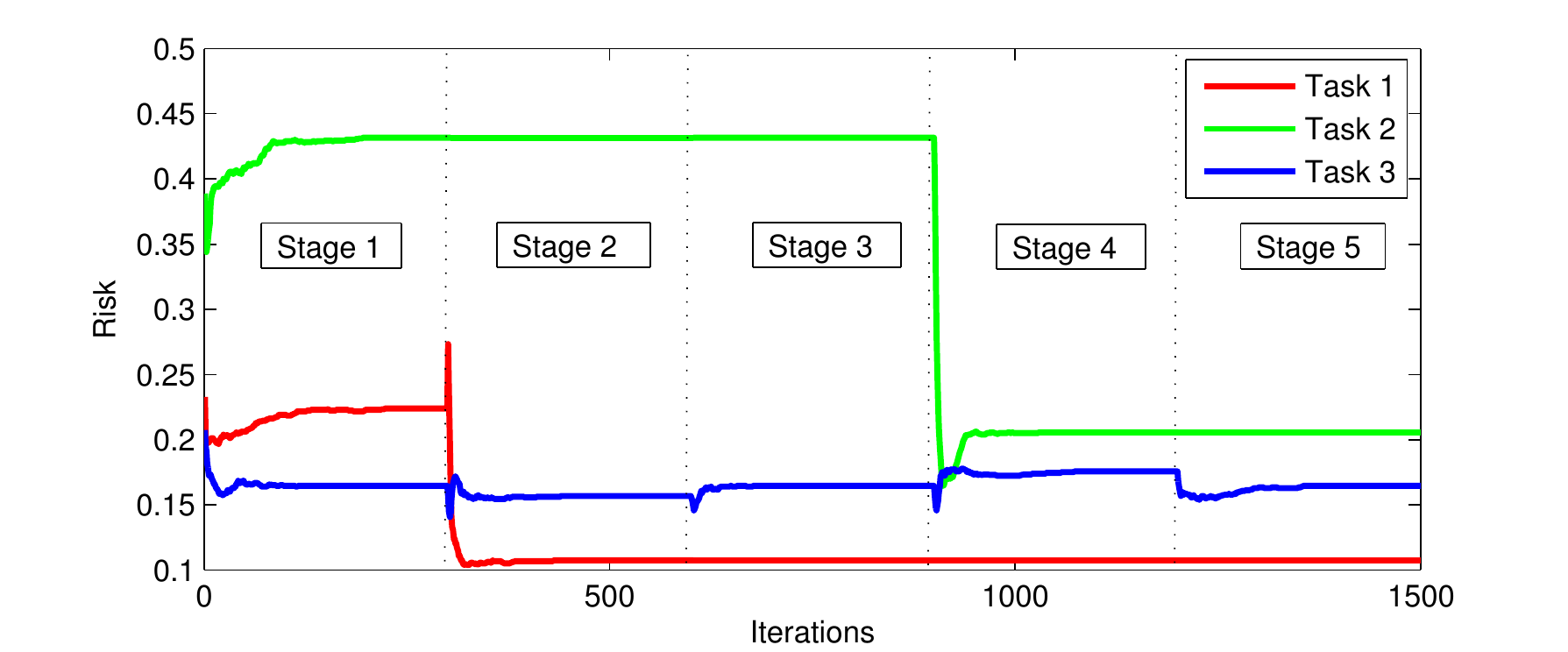}
\vspace{-5mm}
\caption{Evolution of global risks of DTSVM switching with DSVM in real-time. The network is fully connected with $6$ nodes. Each node contains $10$, $10$, $40$ training samples from Task 1, 2 and 3, respectively. Task 1 and 2 are the target tasks, while Task 3 is the source task. In Stage 1, Task 1, 2 and 3 train individually with DSVM;  in Stage 2, Task 1 and 3 train together with DTSVM, Task 2 continues using DSVM; in Stage 3, Task 1 finishes training, Task 2 and 3 train with DSVM; in Stage 4, Task 2 and 3 train together with DTSVM; in Stage 5, Task 2 finishes training, Task 3 trains with DSVM. Note that $\varepsilon_1 = 1$, $\varepsilon_2 = 100$, $\eta_1 = 1$ and $\eta_2 = 1$ and $C = 0.01$.}
\label{fig:Ex5}
\end{figure}
Fig. \ref{fig:Ex5} shows the results of online transfer learning. Task 1 and Task 2 are the target tasks whose risks we aim to reduce, while Task 3 is the source task that can be used to improve the performances of the target tasks. At different stages, Task 1 and Task 2 will enter or leave the DTSVM algorithm with Task 3. Both Task 1 and Task 2 have better performances after training with Task 3. This experiment shows that our DTSVM algorithm can work online without rerunning the whole system.
\section{Conclusion}
In this paper, we have extended a centralized SVM-based transfer learning into a distributed framework. By using ADMoM, we have developed a fully distributed algorithm (DTSVM) where each task in each node operates their own data without transferring training data to other tasks and neighboring nodes. Numerical experiments have shown that our DTSVM algorithm can improve the performances of the target tasks that lack training data or have unbalanced training labels. We have also shown that our algorithm can improve the performances of the nodes who lack the data from the source tasks, by sending information from the nodes who contain the data from the source tasks. We have demonstrated that our algorithm is suitable for online learning where the target tasks can freely enter or leave the training of the source tasks in real-time. One direction of future works is to extend the current framework to nonlinear algorithms and other machine learning algorithms. 
\section*{Appendix A}
Problem (\ref{eq:DistributedTransferMatrix}) can be solved in a distributed way with ADMoM \cite{boyd2011distributed}, which solves the following problem:
\begin{equation}
\label{eq:ADMoM}
\begin{array}{c}
\min\limits_{\{ \mathbf{r}, \omega   \}} F_1(\mathbf{r}) + F_2(\mathbf{\omega}) \\
\begin{array}{lr}
{\text{s.t.}}&{\mathbf{M}\mathbf{r} = \omega.}
\end{array}
\end{array}
\end{equation}
with the following iterations:
\begin{equation}
\label{eq:ADMoMSol1}
{\bf{r}}^{(k + 1)} \in \arg \mathop {\min }\limits_{{\bf{r}}} {F_1}({\bf{r}}) + {\alpha ^{(k)T}}{\bf{Mr}}  + \frac{\eta }{2}{\left\| {{\bf{Mr}} - \omega^{(k)} } \right\|^2},
\end{equation}
\begin{equation}
\label{eq:ADMoMSol2}
\omega^{ (k + 1)} \in \arg \mathop {\min }\limits_{\omega } {F_2}(\omega ) - {\alpha ^{(k)T}}\omega + \frac{\eta }{2}{\left\| {{\bf{M}}\mathbf{r}^{(k+1)} - \omega } \right\|^2},
\end{equation}
\begin{equation}
\label{eq:ADMoMSol3}
\alpha ^{(k + 1)} = \alpha^{ (k)} + \eta ({\bf{Mr}}^{(k + 1)} - \omega^{ (k + 1)}),
\end{equation}
where $\alpha$ denotes the Lagrange multiplier corresponding to the constraint $ \mathbf{Mr} = \omega$. 

We follow a similar step in \cite{forero2010consensus}, by setting \[\mathbf{r}= [\mathbf{r}_{11};...;\mathbf{r}_{1T};...;\mathbf{r}_{V1};...;\mathbf{r}_{VT}] \]
and \[\begin{array}{c}
\omega = [\{\varphi_{1ts}\}_{t,s\in\mathcal{T},s\neq t};...;\{\varphi_{Vts}\}_{t,s\in\mathcal{T},s\neq t};\\  \{ \omega_{vu1} \}_{v\in\mathcal{V},u\in\mathcal{B}_v};...;\{\omega_{vuT}\}_{v\in\mathcal{V},u\in\mathcal{B}_V} ],
\end{array}\]
Problem (\ref{eq:DistributedTransferMatrix}) can be transformed into the form of (\ref{eq:ADMoM}), and thus be solved by Iterations (\ref{eq:ADMoMSol1})-(\ref{eq:ADMoMSol3}). By splitting each iterations into sub-problems and further simplifications, distributed iterations of solving problem (\ref{eq:DistributedTransferMatrix}) can be summarized into the following lemma.
\begin{lemma}
\label{lemmaApp1}
Problem (\ref{eq:DistributedTransferMatrix}) can be solved by the following iterations:
\begin{equation}
\label{eq:AppStep1Soli1}
\{\mathbf{r}_{vt}^{(k+1)},\mathbf{\xi}_{vy}^{(k+1)} \} \in  \arg \min\limits_{\{ \mathbf{r}_{vt},\mathbf{\xi}_{vt}  \}} \mathcal{L} (    \mathbf{r}_{vt}, \mathbf{\xi}_{vy}, \varphi_{vts}^{(k)},\omega_{vut}^{(k)},\alpha_{vts,d}^{(k)},\beta_{vut,d}^{(k)} )
\end{equation}
\begin{equation}
\label{eq:AppStep1Soli2}
 \varphi_{vts}^{(k+1)} \in  \arg \min\limits_{ \varphi_{vts}  } \mathcal{L} (    \mathbf{r}_{vt}^{(k+1)}, \mathbf{\xi}_{vy}^{(k+1)}, \varphi_{vts},\omega_{vut}^{(k)},\alpha_{vts,d}^{(k)},\beta_{vut,d}^{(k)} ),
\end{equation}
\begin{equation}
\label{eq:AppStep1Soli3}
 \omega_{vts}^{(k+1)} \in  \arg \min\limits_{ \omega_{vts}  } \mathcal{L} (    \mathbf{r}_{vt}^{(k+1)}, \mathbf{\xi}_{vy}^{(k+1)}, \varphi_{vts}^{(k)},\omega_{vut},\alpha_{vts,d}^{(k)},\beta_{vut,d}^{(k)} ),
\end{equation}
\begin{equation}
\label{eq:AppStep1Soli4}
\alpha_{vts,1}^{(k+1)} = \alpha_{vts,1}^{(k)} + \eta_1([\mathbf{I,0}]\mathbf{r}_{vt}^{(k+1)} - \varphi_{vts}^{(k+1)}),
\end{equation}
\begin{equation}
\label{eq:AppStep1Soli5}
\alpha_{vts,2}^{(k+1)} = \alpha_{vts,2}^{(k)} + \eta_1( \varphi_{vts}^{(k+1)}-[\mathbf{I,0}]\mathbf{r}_{vs}^{(k+1)} ),
\end{equation}
\begin{equation}
\label{eq:AppStep1Soli6}
\beta_{vut,1}^{(k+1)} = \beta_{vut,1}^{(k)} + \eta_2(\mathbf{r}_{vt}^{(k+1)} - \omega_{vut}^{(k+1)}),
\end{equation}
\begin{equation}
\label{eq:AppStep1Soli7}
\beta_{vut,2}^{(k+1)} = \beta_{vut,2}^{(k)} + \eta_2( \omega_{vut}^{(k+1)}-\mathbf{r}_{ut}^{(k+1)} ),
\end{equation}
where
\begin{equation}
\label{eq:AppStep1L}
\begin{array}{l}
\mathcal{L}(    \mathbf{r}_{vt}, \mathbf{\xi}_{vy}, \varphi_{vts},\omega_{vut},\alpha_{vts,d},\beta_{vut,d} )
=\frac{\epsilon_1}{2} \sum\limits_{v\in\mathcal{V}}\sum\limits_{t\in\mathcal{T}} \mathbf{r}_{vt}^T {\mathbf{M}_1}\mathbf{r}_{vt}  \\ +\frac{\epsilon_2}{2} \sum\limits_{v\in\mathcal{V}}\sum\limits_{t\in\mathcal{T}} \mathbf{r}_{vt}^T {\mathbf{M}_2} \mathbf{r}_{vt}  + VTC  \sum\limits_{v\in\mathcal{V}} \sum\limits_{t\in\mathcal{T}}\mathbf{\xi}_{vt} \\ + \sum\limits_{v\in\mathcal{V}}\sum\limits_{t\in\mathcal{T}} \sum\limits_{s\in\mathcal{T},s\neq t} \left\lbrace \alpha_{vts,1}^T(\left[ \mathbf{I,0}\right] \mathbf{r}_{vt} - \varphi_{vts})  \right\rbrace \\ + \sum\limits_{v\in\mathcal{V}}\sum\limits_{t\in\mathcal{T}} \sum\limits_{s\in\mathcal{T},s\neq t} \left\lbrace   \alpha_{vts,2}^T(\varphi_{vts}-\left[ \mathbf{I,0}\right]\mathbf{r}_{vs}) \right\rbrace \\ + \sum\limits_{v\in\mathcal{V}}\sum\limits_{u\in\mathcal{B}_v}\sum\limits_{t\in\mathcal{T}} \left\lbrace \beta_{vut,1}^T(\mathbf{r}_{vt} - \omega_{vut})+\beta_{vut,2}^T(\omega_{vut} - \mathbf{r}_{ut})\right\rbrace\\ + \frac{\eta_1}{2}\sum\limits_{v\in\mathcal{V}}\sum\limits_{t\in\mathcal{T}} \sum\limits_{s\in\mathcal{T},s\neq t} \left\lbrace \parallel \left[ \mathbf{I,0}\right] \mathbf{r}_{vt} - \varphi_{vts}\parallel_2^2  + \parallel\varphi_{vts}-\left[ \mathbf{I,0}\right]\mathbf{r}_{vs}\parallel_2^2 \right\rbrace \\ +\frac{\eta_2}{2} \sum\limits_{v\in\mathcal{V}}\sum\limits_{u\in\mathcal{B}_v}\sum\limits_{t\in\mathcal{T}} \left\lbrace \parallel \mathbf{r}_{vt} - \omega_{vut}\parallel_2^2 +\parallel \omega_{vut} - \mathbf{r}_{ut}\parallel_2^2\right\rbrace .
\end{array}
\end{equation}
\end{lemma}
Setting initial conditions $\alpha_{vts,1}^{(0)}=\alpha_{vts,2}^{(0)}=\mathbf{0}_{(p+1)\times 1}$ and $\beta_{vut,1}^{(0)}=\beta_{vut,2}^{(0)}=\mathbf{0}_{(2p+2)\times 1}$, we have $\alpha_{vts,1}^{(k)}=\alpha_{vts,2}^{(k)}$ and $\beta_{vut,1}^{(k)}=\beta_{vut,2}^{(k)}$ for $k \geq 0 $. We further define $\alpha_{vt} =\sum_{s\in\mathcal{T},s\neq t} \alpha_{vts,1}$ and $\beta_{vt} = \sum_{u\in\mathcal{B}_v} \beta_{vut,1}$. Note that, $ \varphi_{vts} = \frac{1}{2}[ \mathbf{I,0}]( \mathbf{r}_{vt} + \mathbf{r}_{vs}  ) $,
and $  \omega_{vut} = \frac{1}{2} ( \mathbf{r}_{vt} + \mathbf{r}_{ut}  )$, 
which can be solved directly from (\ref{eq:AppStep1Soli2}) and (\ref{eq:AppStep1Soli3}). With further simplification, iterations (\ref{eq:AppStep1Soli1})-(\ref{eq:AppStep1Soli7}) can be simplified as the following lemma. 
\begin{lemma}
\label{lemmaApp2}
With $\alpha_{vt}^{(0)}=\mathbf{0}_{(p+1)\times 1}$ and $\beta_{vt}^{(0)}=\mathbf{0}_{(2p+2)\times 1}$, iterations (\ref{eq:AppStep1Soli1})-(\ref{eq:AppStep1Soli7}) can be reduced into the following iterations:
\begin{equation}
\label{eq:AppStep2Soli1}
\{\mathbf{r}_{vt}^{(k+1)},\mathbf{\xi}_{vy}^{(k+1)} \} \in  \arg \min\limits_{\{ \mathbf{r}_{vt},\mathbf{\xi}_{vt}  \}} \mathcal{L}' (    \mathbf{r}_{vt}, \mathbf{\xi}_{vy}, \alpha_{vt}^{(k)},\beta_{vt}^{(k)} ),
\end{equation}
\begin{equation}
\label{eq:AppStep2Soli2}
\alpha_{vt}^{(k+1)} = \alpha_{vt}^{(k)} + \frac{\eta_1}{2}[\mathbf{I,0}]\sum\limits_{s\in\mathcal{T},s\neq t}(\mathbf{r}_{vt}^{(k+1)}-\mathbf{r}_{vs}^{(k+1)}),
\end{equation}
\begin{equation}
\label{eq:AppStep2Soli3}
\beta_{vt}^{(k+1)} = \beta_{vt}^{(k)} + \frac{\eta_2}{2}\sum\limits_{u\in\mathcal{B}_v}(\mathbf{r}_{vt}^{(k+1)}-\mathbf{r}_{ut}^{(k+1)}),
\end{equation}
where 
\begin{equation}
\label{eq:Appstep2L}
\begin{array}{l}
\mathcal{L}' (    \mathbf{r}_{vt}, \mathbf{\xi}_{vy}, \alpha_{vt},\beta_{vt} )
=\frac{\epsilon_1}{2} \sum\limits_{v\in\mathcal{V}}\sum\limits_{t\in\mathcal{T}} \mathbf{r}_{vt}^T {\mathbf{M}_1}\mathbf{r}_{vt}  \\ +\frac{\epsilon_2}{2} \sum\limits_{v\in\mathcal{V}}\sum\limits_{t\in\mathcal{T}} \mathbf{r}_{vt}^T {\mathbf{M}_2} \mathbf{r}_{vt}  + VTC  \sum\limits_{v\in\mathcal{V}} \sum\limits_{t\in\mathcal{T}}\mathbf{\xi}_{vt} \\ + 2\sum\limits_{v\in\mathcal{V}}\sum\limits_{t\in\mathcal{T}}  \alpha_{vt}^T\left[ \mathbf{I,0}\right] \mathbf{r}_{vt}  + 2\sum\limits_{v\in\mathcal{V}}\sum\limits_{t\in\mathcal{T}}  \beta_{vt}^T\mathbf{r}_{vt} \\ + \eta_1 \sum\limits_{v\in\mathcal{V}}\sum\limits_{t\in\mathcal{T}} \sum\limits_{s\in\mathcal{T},s\neq t}  \left|\left| \left[ \mathbf{I,0}\right] \mathbf{r}_{vt} - \frac{1}{2}[ \mathbf{I,0}]( \mathbf{r}_{vt}^{(k)} + \mathbf{r}_{vs}^{(k)}  ) \right|\right|_2^2   \\ + \eta_2\sum\limits_{v\in\mathcal{V}}\sum\limits_{u\in\mathcal{B}_v}\sum\limits_{t\in\mathcal{T}} \left|\left| \mathbf{r}_{vt} - \frac{1}{2} ( \mathbf{r}_{vt}^{(k)} + \mathbf{r}_{ut}^{(k)}) \right|\right|_2^2 .
\end{array}
\end{equation}
\end{lemma}
Introducing unused constraints $\mathbf{Y}_{vt}\mathbf{X}_{vt}[\mathbf{I,I}]\mathbf{r}_{vt} \succeq \mathbf{1}_{vt} - \mathbf{\xi}_{vt}$ and $\mathbf{\xi}_{vt} \succeq \mathbf{0}_{vt}$ with Lagrangian multipliers $\lambda_{vt}$ and $\gamma_{vt}$ into (\ref{eq:AppStep2Soli1}),
by KKT conditions, we can achieve:
\begin{equation}
\label{eq:AppStep31}
\begin{array}{c}
\mathbf{r}_{vt}=\mathbf{U}_{vt}^{-1} \bigg( [\mathbf{I,I}]^T\mathbf{X}_{vt}^T\mathbf{Y}_{vt}\lambda_{vt} - 2[\mathbf{I,0}]^T\alpha_{vt}  -2\beta_{vt}\\ \ \ \ \ \ +\eta_1 \sum\limits_{s\in\mathcal{T},s\neq t}[\mathbf{I,0}]^T[\mathbf{I,0}](\mathbf{r}_{vt}^{(k)}+\mathbf{r}_{vs}^{(k)}) + \eta_2 \sum\limits_{u\in\mathcal{B}_v} (\mathbf{r}_{vt}^{(k)}+\mathbf{r}_{ut}^{(k)}) \bigg).
\end{array}
\end{equation}
\begin{equation}
\label{eq:AppStep32}
VTC\mathbf{1}_{vt}-\lambda_{vt}-\gamma_{vt}=\mathbf{0}_{vt}.
\end{equation}
Note that,
\begin{equation}
\label{eq:AppStep3U}
\mathbf{U}_{vt} = \epsilon_1\mathbf{M}_1 + \epsilon_2 \mathbf{M}_2 + 2\eta_1(T-1)[\mathbf{I,0}]^T[\mathbf{I,0}]+2\eta_2|\mathcal{B}_v|\mathbf{I}_{2p+2}.
 \end{equation}
Letting
\begin{equation}
\label{eq:AppStep3f}
\begin{array}{l}
\mathbf{f}_{vt} =  2[\mathbf{I,0}]^T\alpha_{vt}  + 2\beta_{vt}\\ \ \ \ \ \ -\eta_1 \sum\limits_{s\in\mathcal{T},s\neq t}[\mathbf{I,0}]^T[\mathbf{I,0}](\mathbf{r}_{vt}^{(k)}+\mathbf{r}_{vs}^{(k)}) - \eta_2 \sum\limits_{u\in\mathcal{B}_v} (\mathbf{r}_{vt}^{(k)}+\mathbf{r}_{ut}^{(k)}),
\end{array}
\end{equation}
we can also achieve:
\begin{equation}
\label{eq:AppStep3Lambda}
\begin{array}{l}
\lambda_{vt} \in\arg\max\limits_{\lambda_{vt}} -\frac{1}{2}\lambda_{vt}^T\mathbf{Y}_{vt}\mathbf{X}_{vt}[\mathbf{I,I}]\mathbf{U}_{vt}^{-1}[\mathbf{I,I}]^T\mathbf{X}_{vt}^T\mathbf{Y}_{vt}\lambda_{vt} \\ \ \ \ \ \ \ \ \ \ \ \ \ \ \ \ \ + (\mathbf{1}_{vt}+\mathbf{Y}_{vt}\mathbf{X}_{vt}[\mathbf{I,I}]\mathbf{U}_{vt}^{-1}\mathbf{f}_{vt})^T\lambda_{vt}.
\end{array}
\end{equation}
Thus, iterations of solving Problem (\ref{eq:DistributedTransferMatrix}) can be summarized as Proposition 1.

Note that, since $\mathbf{M}_1$ and $\mathbf{M}_2$ are semi-positive matrices, the objective function of Problem (\ref{eq:DistributedTransferMatrix}) can be shown that it is closed, proper, and convex. Moreover, it is easy to see that the unaugmented Lagrangian $\mathcal{L}$ in (\ref{eq:AppStep1L}) has a saddle point which satisfies \[\begin{array}{c}
\mathcal{L}(    \mathbf{r}_{vt}^*, \mathbf{\xi}_{vy}^*, \varphi_{vts}^*,\omega_{vut}^*,\alpha_{vts,k},\beta_{vut,k} ) \\ \leq \mathcal{L}(    \mathbf{r}_{vt}^*, \mathbf{\xi}_{vy}^*, \varphi_{vts}^*,\omega_{vut}^*,\alpha_{vts,k}^*,\beta_{vut,k}^* )\\ \leq \mathcal{L}(    \mathbf{r}_{vt}, \mathbf{\xi}_{vy}, \varphi_{vts},\omega_{vut},\alpha_{vts,k}^*,\beta_{vut,k}^* ).
\end{array}\]
Thus, iterations in Proposition 1 converge to the solution of Problem (\ref{eq:DistributedTransferMatrix}) based on Section 3.2 and Appendix A in \cite{boyd2011distributed}.

\bibliographystyle{ieeetr}
\bibliography{DistributedTransferZhang.bib}

\begin{thebibliography}{10}

\bibitem{osuna1997training}
E.~Osuna, R.~Freund, and F.~Girosi, ``Training support vector machines: an
  application to face detection,'' in {\em Computer vision and pattern
  recognition, 1997. Proceedings., 1997 IEEE computer society conference on},
  pp.~130--136, IEEE, 1997.

\bibitem{mccallum1999machine}
A.~McCallum, K.~Nigam, J.~Rennie, and K.~Seymore, ``A machine learning approach
  to building domain-specific search engines,'' in {\em IJCAI}, vol.~99,
  pp.~662--667, Citeseer, 1999.

\bibitem{shao2015transfer}
L.~Shao, F.~Zhu, and X.~Li, ``Transfer learning for visual categorization: A
  survey,'' {\em Neural Networks and Learning Systems, IEEE Transactions on},
  vol.~26, no.~5, pp.~1019--1034, 2015.

\bibitem{pan2010survey}
S.~J. Pan and Q.~Yang, ``A survey on transfer learning,'' {\em Knowledge and
  Data Engineering, IEEE Transactions on}, vol.~22, no.~10, pp.~1345--1359,
  2010.

\bibitem{dai2007boosting}
W.~Dai, Q.~Yang, G.-R. Xue, and Y.~Yu, ``Boosting for transfer learning,'' in
  {\em Proceedings of the 24th international conference on Machine learning},
  pp.~193--200, ACM, 2007.

\bibitem{evgeniou2004regularized}
T.~Evgeniou and M.~Pontil, ``Regularized multi--task learning,'' in {\em
  Proceedings of the tenth ACM SIGKDD international conference on Knowledge
  discovery and data mining}, pp.~109--117, ACM, 2004.

\bibitem{forero2010consensus}
P.~A. Forero, A.~Cano, and G.~B. Giannakis, ``Consensus-based distributed
  support vector machines,'' {\em The Journal of Machine Learning Research},
  vol.~11, pp.~1663--1707, 2010.

\bibitem{boyd2011distributed}
S.~Boyd, N.~Parikh, E.~Chu, B.~Peleato, and J.~Eckstein, ``Distributed
  optimization and statistical learning via the alternating direction method of
  multipliers,'' {\em Foundations and Trends{\textregistered} in Machine
  Learning}, vol.~3, no.~1, pp.~1--122, 2011.

\bibitem{vapnik2013nature}
V.~Vapnik, {\em The nature of statistical learning theory}.
\newblock Springer Science \& Business Media, 2013.

\bibitem{ben2003exploiting}
S.~Ben-David and R.~Schuller, ``Exploiting task relatedness for multiple task
  learning,'' in {\em Learning Theory and Kernel Machines}, pp.~567--580,
  Springer, 2003.

\bibitem{MNIST}
``Statistical power analysis software.''
\newblock http://yann.lecun.com/exdb/mnist/.

\bibitem{jolliffe2002principal}
I.~Jolliffe, {\em Principal component analysis}.
\newblock Wiley Online Library, 2002.

\end{thebibliography}

\end{document}